\documentclass[letterpaper, 10 pt, conference]{ieeeconf}  

\IEEEoverridecommandlockouts                              

\usepackage{mathrsfs}
\usepackage{amsmath}
\usepackage{amssymb}
\usepackage{algorithm}
\usepackage[noend]{algpseudocode}
\usepackage{color}
\usepackage{cite}
\usepackage{subfigure}
\usepackage{bbm}
\usepackage{url}

\newtheorem{theorem}{Theorem}
\newtheorem{lemma}{Lemma}
\newtheorem{proposition}{Proposition}

\usepackage{graphicx}
\newcommand{\depth}{$H$}
\newcommand{\depthM}{H}
\newcommand{\diam}{D}
\newcommand{\algo}{$\textsc{Circle-Tag}$}
\newcommand{\unitd}{r}
\newcommand{\block}{b_{i,k}}
\newcommand{\prob}{ONFP}

\algnewcommand\algorithmicswitch{\textbf{switch}}
\algnewcommand\algorithmiccase{\textbf{case}}
\algdef{SE}[SWITCH]{Switch}{EndSwitch}[1]{\algorithmicswitch\ #1\ \algorithmicdo}{\algorithmicend\ \algorithmicswitch}%
\algdef{SE}[CASE]{Case}{EndCase}[1]{\algorithmiccase\ #1}{\algorithmicend\ \algorithmiccase}%
\algtext*{EndSwitch}%
\algtext*{EndCase}%
\algnewcommand\algorithmicforeach{\textbf{for each}}
\algdef{S}[FOR]{ForEach}[1]{\algorithmicforeach\ #1\ \algorithmicdo}

\title{\LARGE \bf
Asynchronous Network Formation in \\Unknown Unbounded Environments*
}

\author{Selim Engin$^{1}$ and Volkan Isler$^{1}$
\thanks{*This work was supported by the NSF grant \#1617718 and a Minnesota State LCCMR grant}
\thanks{$^{1}$Selim Engin and Volkan Isler are with Department of Computer Science and Engineering,
        University of Minnesota, MN 55455, USA
        {\tt\small \{engin003, isler\}@umn.edu}}%
}

\begin{document}

\maketitle
\thispagestyle{empty}
\pagestyle{empty}

\begin{abstract}
In this paper, we study the Online Network Formation Problem (ONFP) for a mobile multi-robot system. Consider a group of robots with a bounded communication range operating in a large open area. One of the robots has a piece of information which has to be propagated to all other robots. What strategy should the robots pursue to disseminate the information to the rest of the robots as quickly as possible?
The initial locations of the robots are unknown to each other, therefore the problem must be solved in an online fashion.

For this problem, we present an algorithm whose competitive ratio is $O(H \cdot \max\{M,\sqrt{M H}\})$ for arbitrary robot deployments, where $M$ is the largest edge length in the Euclidean minimum spanning tree on the initial robot configuration and $H$ is the height of the tree.
We also study the case when the robot initial positions are chosen uniformly at random and improve the ratio to $O(M)$.
Finally, we present simulation results to validate the performance in larger scales and demonstrate our algorithm using three robots in a field experiment.

\end{abstract}


\section{Introduction}
Consider a group of robots  with unknown locations scattered over a large area.
The robots can not communicate with each other unless they are within a given communication range.
Such scenarios may occur, for example, after aerial deployments in disaster response situations.
It might also be the case that many robots operate independently in a search task without coordination and disperse across the environment.

Now imagine that one of the robots wants to convey a piece of information to all other robots.
This information can be used for network formation, rendezvous or any other task.
For this purpose, the robot starts searching for other robots. Once a robot is found, it can be recruited to propagate the information.
What strategy should the robot team follow to convey this information as soon as possible?
In this paper, we study this online network formation in unbounded environments.

Our problem is closely related to the Freeze-Tag Problem (FTP), in which there is initially a single \textit{active} (leader) robot and the rest of the robots are \textit{frozen/asleep}, staying put~\cite{arkin2006freeze}. 
The goal is to wake up all the asleep robots in the shortest time.
An asleep robot awakens when an active robot is in its close proximity.
After a robot awakens, it becomes activated and participates in the process of rousing the robots.

The network formation problem can be regarded as the online version of the original FTP, where the robot positions are unknown to each other. 
Without the knowledge of the locations of their peers, the active robots need to perform a search strategy to discover them as illustrated in Figure~\ref{fig:uavFTP}.
Since the robots do not necessarily start executing the algorithm at the same time, our problem is asynchronous. This property of the problem is beneficial especially in real-world systems where exact synchrony is hard to achieve.

\begin{figure}[ht!]
\centering
\includegraphics[width=0.75\columnwidth]{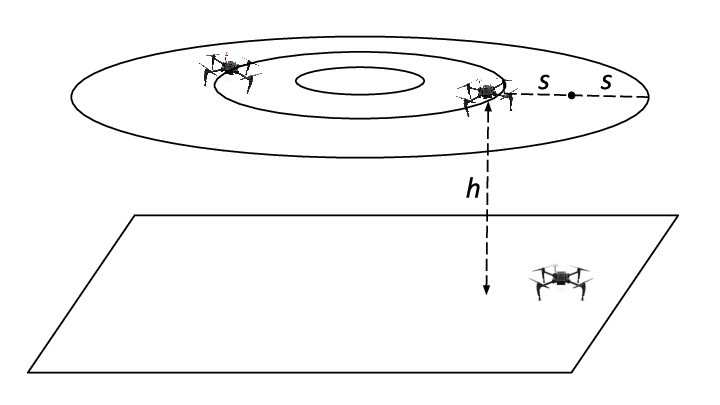}
\caption{An example scenario where two active robots are searching for the frozen robot on the ground. The robot altitude is $h$ and step size is $s$.}
\label{fig:uavFTP}
\end{figure}

Algorithm design for mobile network formation has received significant attention over the past years.
These approaches can be categorized as centralized and decentralized solutions~\cite{zavlanos2011graph}.
In a centralized approach, a base station or a leader robot has access to positions of all the robots and can directly command them~\cite{zavlanos2005controlling,srivastava2008multi,engin2018minimizing}.
On the other hand, with decentralized algorithms the robots can perform tasks using local information, without explicit instructions from an external source~\cite{cortes2006robust,ji2007distributed}.

While centralized approaches are usually efficient and easy to implement, adopting these solutions may not be possible in settings where the communication range of the robots is short.
Consequently, there is an increasing interest in distributed communication protocols.
Gossip protocols, for instance, are designed to disseminate information in a connected network of nodes \cite{boyd2006randomized}.
In the single-piece dissemination scenario of a gossip protocol, an arbitrarily chosen node has a piece of information and the objective is to spread this information to all other nodes in the shortest time \cite{shah2009gossip}. In this one-to-all broadcast setting, the network is stationary and a node propagates the information to its neighbors with an associated probability. 
Contrary to this formulation, in the FTP the nodes are able to move and transfer the piece of information to a neighbor as soon as they are within each other's connectivity range.

In the rendezvous problem, a group of robots need to meet at a point as quickly as possible.
The two-player case with initial positions unknown to each other was studied in \cite{alpern1995rendezvous}.
Roy and Dudek proposed strategies to achieve rendezvous together with the exploration of an unknown environment \cite{roy2001collaborative}.
Ozsoyeller et al. studied the rendezvous search problem on the line \cite{ozsoyeller2013symmetric} and in planar environments \cite{ozsoyeller2019rendezvous}. 
Robots in this formulation do not know the positions of each other and they have to execute the same algorithm. 
Due to this constraint, the strategy is inherently stochastic since otherwise the robots cannot meet by following the same motion.
An asymmetric version of this problem where a robot needs to rendezvous with non-adversarial targets is analyzed in \cite{meghjani2016multi}.

Poduri and Sukhatme studied the problem of establishing a connected network in a bounded toroidal domain using a decentralized approach \cite{poduri2007latency}.
In their model the robots perform random walk, and stop moving when they are near the base station or a robot who is already connected.

The FTP was originally posed in the graph settings where the robots are placed on the vertices of a graph \cite{arkin2006freeze}. The robots can move only along the edges and an asleep robot awakens when its vertex is visited by an active robot. In this formulation, each robot knows the positions of all other robots. 
The online version of the problem was studied by \cite{hammar2006online} who proposed a lower bound on the competitive ratio for graph settings.

The online FTP in the Euclidean domain was further explored in \cite{meisner2009probabilistic}. The authors considered the case where the robots are distributed uniformly at random over a rectangular shaped bounded area. In this case, the robots cannot go outside the closed area. For this problem the authors proposed a logarithmic factor approximation algorithm which depends on the area, therefore is not directly applicable.

The results in \cite{meisner2009probabilistic} rely on a bounded environment assumption. Considering the typical operation environments of field robots, including open seas and skies, this constraint may need to be relaxed.
In this paper, we generalize the environment by removing the closed boundary assumption.
In summary, our contributions can be stated as follows.

\textbf{Our contribution:}
We study a scenario where the robots are deployed in an unbounded environment and the robot positions are unknown to each other.
We analyze two cases of the problem: the initial robot positions are chosen arbitrarily and uniformly at random.
For arbitrary deployments we present a strategy which is $O(\depthM{} \cdot \max\{M,\sqrt{M \depthM{}}\})$-competitive. This result improves the competitive ratio of a naive algorithm by a factor of $\depthM{}$ or $\sqrt{M \depthM{}}$. The algorithm performs better in random deployments and the ratio becomes $O(M)$.
In addition to providing performance guarantees, we demonstrate our algorithm in field experiments as a proof-of-concept implementation.

The remainder of the paper is organized as follows.
We formulate our problem and introduce the models used throughout the paper in Section~\ref{sec:form}.
In Section~\ref{sec:alg} we present our algorithm and analyze its performance in Section~\ref{sec:anly}.
Section~\ref{sec:sim} contains simulations we conducted.
We describe the outdoor experiments demonstrating the algorithm in Section~\ref{sec:exp}.
Finally, we conclude with a summary and future research directions in Section~\ref{sec:conc}.

\section{Problem Formulation and Models}
\label{sec:form}
In the Online Network Formation Problem (\prob{}), at the beginning there is a single active robot and the goal is to find a strategy to explore all inactive robots such that the time it takes to tag the last robot is minimized.

In this problem, the robot initial positions are unknown to each other, thus the robots use an online strategy.
The performance of online algorithms are usually evaluated in comparison to the optimal offline solution \cite{motwani2010randomized}.
The optimal offline strategy has access to the initial positions of all robots beforehand, so the leader robot does not need to perform a search but instead can directly move towards the inactive robots.
For an instance $\sigma$ of the problem, suppose $A(\sigma)$ is the solution produced by our online algorithm, and $OPT(\sigma)$ is the optimal offline solution. The algorithm $A$ is said to be $c$-competitive if for all instances of the problem the ratio $A(\sigma)/OPT(\sigma) \leq c$ holds.


\subsection{Environment and Robot Models}
In the \prob, we have $n$ robots scattered over an unbounded area. 
Initially, one active robot carries a piece of information to propagate to all other robots.
An asleep robot receives this information (or awakens), once an active robot is within its close proximity.
The robots are assumed to be point robots that always move at constant speed. 

The initial positions of the robots are $\mathcal{X} = \{x_1, \dots, x_n\}$, where each $x_i \in \mathbb{R}^2$. Without loss of generality we assume $x_1$ is the leader's initial position and $x_2$ to $x_n$ are labeled in non-decreasing order with respect to their distances to $x_1$. 
The Euclidean distance between two robots $x_i$ and $x_j$ is denoted by $d(x_i,x_j)$.	
In our analysis, we use the Minimum Spanning Tree (MST) on the initial robot positions and denote it by $\mathcal{T}$.
We denote the longest edge length in $\mathcal{T}$ as $M$, and its height as \depth{}.

Finally, the communication range is denoted by $\unitd{}$. We measure the distances in units of $\unitd{}$, and a robot displaces by $\unitd{}$ per time step.
In the performance analysis of our algorithm, we consider two cases: sparse and dense configurations. A configuration is called sparse if $n \cdot \unitd{} \leq M$, and dense otherwise.
In the rest of the paper we will ignore the unit distance in the expressions and regard as $\unitd{} = 1$.
 

\section{The Proposed Algorithm}
\label{sec:alg}
In this section, we describe \algo{}, our strategy for the \prob{}.
In essence, the algorithm partitions the unexplored environment into blocks of equal size and assigns an active robot to each of them as shown in Figure~\ref{fig:circletag}.
The blocks are shaped as cut circle sectors and each active robot performs an arc-shaped back-and-forth motion to cover the block. 
This motion guarantees that an asleep robot is going to be explored when it is inside a block.

\begin{figure}[ht!]
\centering
\includegraphics[width=0.7\columnwidth]{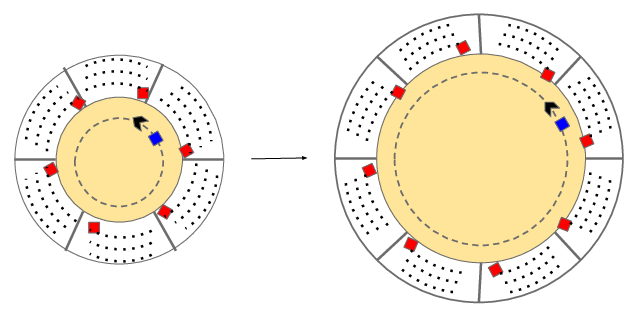}
\caption{\algo{} partitions the area into blocks of equal size and assigns an active robot (red) to each.
The leader (blue) collects the explored robots and assigns blocks in the next round.}
\label{fig:circletag}
\end{figure}

The algorithm includes a set of behaviors which we introduce next.
The execution of our algorithm consists of a sequence of \textbf{rounds}.
A round starts when the leader begins traveling on its circle and ends when the last active robot finishes covering its assigned block for that round.

The motion of the leader robot is called \textbf{Concentric-Circles-Search (CCS)}, and CCS($i$) refers to the path of the leader in round $i$.
The CCS is composed of concentric circles whose radii increase at each round. These circles are always centered at the initial position of the leader robot.
Suppose the leader travels on a circle of radius $R_i$ during a round $i$. If the number of activated robots is $k$, then $R_i$ is given by:

\begin{equation*}
R_i = \begin{cases}
R_{i-1} + 1, & \text{$k=0$ or $k=1$}\\
R_{i-1} + k, & k \geq 2
\end{cases}
\end{equation*}

where the starting radius $R_0$ is zero.
Initially, there is only the leader and no other robot is active, so $k=0$.
Until finding a robot, at each round the leader updates the radius $R_i$ by 1. This motion pattern ensures a full coverage of the leader's surrounding.

At the end of each round, the leader moves in the East direction until it reaches the circle for the next round.
Therefore, CCS($i$) in a round $i$ consists of traveling on a circle of radius $R_i$ and moving by $R_{i+1} - R_i$ towards East to get to the next circle.
An example of this motion is illustrated in Figure~\ref{fig:concentric}. Here, the leader robot is shown in blue. The yellow area is the already covered part and gray area is the region explored in the current circle. The white outer part of the environment is unknown to the robot.


When the leader hits its nearest frozen robot, the activated robot does not immediately participate in the search, but merely follows the leader robot instead.
After the second nearest robot is found, the algorithm proceeds by partitioning the area into two blocks and assigns an active robot for each.
In each round, the active robots other than the leader cover their blocks in multiple back-and-forth sweeps. This motion pattern is called \textbf{Block-Cover} and an instance of it is shown in Figure~\ref{fig:coverage}. 
The actives are reassigned to their new blocks in the beginning of the next round.

\begin{figure}
\centering     
\subfigure[CCS \label{fig:concentric}]{\includegraphics[width=0.3\columnwidth]{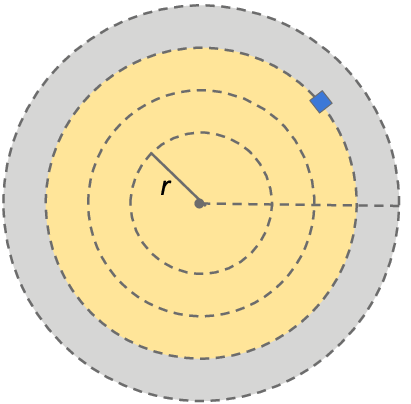}}
\subfigure[Block-Cover \label{fig:coverage}]{\includegraphics[width=0.5\columnwidth]{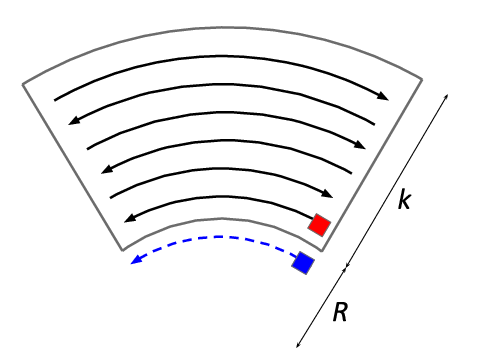}}
\caption{(a) CCS: The leader robot travels along a circle at each round. (b) Block-Cover: Each active robot covers its block in multiple sweeps.}

\end{figure}

Suppose that there are $k_i \geq 2$ active robots, excluding the leader, in the beginning of round $i$.
Then, the number of blocks for that round is also $k_i$.
Let $\block{}$ denote a block in round $i$.
In each block $\block{}$ the active robot travels along $k_i$ concentric arcs and brings all the robots it encounters together, until the round finishes.
The task of the leader robot is to circle around the covered area and reassign the active robots to their blocks, taking the newly found robots into account.

In round $i+1$, the leader robot starts by collecting the robots discovered in round $i$, as it passes from the active robot locations.
An active robot starts sweeping the arcs as soon as it receives the block assignment.
The leader robot keeps the collected robots until the end of the round and assigns blocks for them in the next round.
Therefore, a robot explored in round $i$ joins the search starting from round $i+2$.
The algorithm proceeds in this manner by expanding the coverage area, and terminates when all robots are discovered.

%


\section{Performance of \algo{}}
\label{sec:anly}
We proceed with analyzing the performance of \algo{}. We start with the case where the initial positions of the robots are chosen arbitrarily, then continue with the case when they are distributed uniformly at random.

\subsection{Sweeping a Block}
The algorithm \algo{} begins the process of exploring the frozen robots starting from the leader's initial position.
After the leader awakens two robots, each activated robot works on a dedicated zone assigned by the leader.

Suppose there are already explored $k_i$ robots in round $i$.
Thus, the number of blocks is also $k_i$.
An active robot sweeps its block $\block{}$ in $k_i$ arcs in increasing radius whose apex angle is $2\pi/k_i$ (see Figure~\ref{fig:coverage}).
The asleep robots encountered during the coverage follow the active robot until the end of the round and become activated in the next round.


\begin{proposition}
\label{prop:roundTime}
Consider a round $i$ with $k_i$ active robots, and let $C(\block{})$ denote the coverage time of a block $\block{}$ using Block-Cover.
The time it takes to complete the round is no more than $2C(\block{})$.
\end{proposition}

%
%
%

The proof is presented in an accompanying technical report \cite{techrep2018nf}.
Proposition~\ref{prop:roundTime} tells us that the execution time of a round is bounded by twice the size of a block in that round.

\subsection{Exploration with \algo{}}
We continue by analyzing the performance of \algo{} in unbounded environments.
Our analysis consists of two cases for the initial configuration of the robots. We say a configuration is sparse if $n \leq M$, and dense if $n > M$.
We start with the sparse case.

\subsubsection{Sparse configurations}
In the beginning of the execution, the leader robot starts by performing the CCS centered at its initial position as shown in Figure~\ref{fig:concentric}.
Consider the MST $\mathcal{T}$ on the initial configuration of the robots, and let $M$ denote the longest edge length in $\mathcal{T}$.
Suppose $x_1$ is the leader robot, and $x_2$ and $x_3$ are the first and second activated robots, respectively. 
Since $\mathcal{T}$ spans all the robots, we know that the nearest neighbor distance of a robot cannot be more than $M$.
Therefore, $d(x_1,x_2) \leq M$, and from the triangle inequality $d(x_1,x_3) \leq 2M$.
Then, the time it takes to explore the first two robots is at most $4\pi M^2$.

To upper bound the solution of our algorithm, that is the time it takes to explore all robots using \algo{}, we consider the worst case.
The worst case configuration for the \prob{} is when the robots are initially in a path configuration with an equal spacing of $M$, and the leader robot is the leftmost or rightmost robot.
This means that the tree $\mathcal{T}$ is a path with all edges of length $M$.
This configuration is the worst case scenario because it has the maximum furthest pairwise distance among all possible configurations.

The analysis of our algorithm comprises multiple \textbf{phases}. 
Each phase is a collection of rounds, and the $t$-th phase finishes at the end of a round where the leader robot reaches a distance $t \cdot M$ away from its initial location.
This ensures that there is a newly activated robot at the end of each phase.
Note that the phases are only used in the analysis and they are not explicitly available to the robots.

For any two consecutive phases $t-1$ and $t$, consider the circles they started and let their radii be $R_{t-1}$ and $R_t$, respectively (see Figure~\ref{fig:phases}).
We define the radius difference as $\Delta^- := R_t - R_{t-1}$, and sum as $\Delta^+ := R_t + R_{t-1}$.

\begin{figure}[ht!]
\centering
\includegraphics[width=0.7\columnwidth]{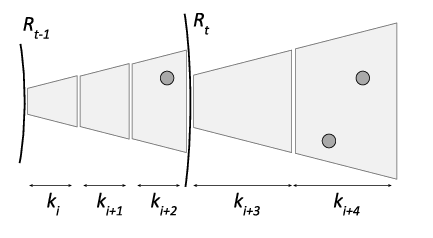}
\caption{Two consecutive phases $t-1$ and $t$ are shown. A phase finishes when there is an activated robot at the end of a round in that phase.}
\label{fig:phases}
\end{figure}

Let $k_t$ be the number of active robots except the leader in phase $t$.
The radius difference $\Delta^-$ cannot be more than $2M$.
This is because the distance between a robot found in phase $t$ to an already explored robot in earlier phases is at most $M$, and since there are at most $k_t$ extra sweeps in a phase after finding a robot, the radii difference $\Delta^-$ is no more than $M + k_t < M + n \leq 2M$. The last inequality holds because $n \leq M$ in sparse configurations.

The radius sum $\Delta^+$ can be bounded by $2M(t-1) + 2Mt < 4Mt$, since the radius of the starting circle of a phase increases by at most $\Delta^- < 2M$.

In each round of phase $t$, the blocks have height $k_t$ since the number of activated robots remains same within a phase and using Block-Cover the active robots sweep $k_t$ arcs.
Therefore, until the next robot is explored, in phase $t$ there will be at most $\Delta^-/k_t < 2M/k_t$ rounds in the worst case. 

Suppose $\diam{}$ is the diameter of the tree $\mathcal{T}$.
The distance between the initial position of the leader robot and its furthest neighbor is no more than $M \diam{}$.
Then, there are at most $2\depthM{}$ phases in the worst case, since $M \diam{}/M = \diam{} \leq 2\depthM{}$.

We can then bound the solution of \algo{} as:

\begin{equation}
\label{eq:sol}
\begin{split}
SOL & \leq 4\pi M^2 + \sum_{t=1}^{2\depthM{}}{2\pi \frac{(R_t^2 - R_{t-1}^2)}{k_t}} \\
	& < 4\pi M^2 + \sum_{t=1}^{2\depthM{}}{2\pi \frac{8M^2t}{t}} = M^2 (32\pi \depthM{} + 4\pi),
\end{split}
\end{equation}

where the radius $R_0$ is zero, and $k_t \geq t$ since in each phase there is at least one explored robot.
For sparse configurations, we then arrive at the following result.

\begin{lemma}
\label{lem:solSparse}
Suppose the MST on the initial configuration has height \depth{} and its maximum edge length is $M$.
Then, the time to explore all robots using \algo{} in sparse configurations is $O(M^2 \depthM{})$.
\end{lemma}


\subsubsection{Dense configurations}
In dense configurations, $n$ is larger than $M$ so we cannot bound the completion time of a phase as before.
However, since the robots use the same strategy in both cases, for the part when the number of activated robots $k_t$ in phase $t$ is smaller than $M$, the performance analysis will be the same.

\begin{lemma}
\label{lem:solDense}
The time it takes to explore all robots in dense configurations is $O(\max \{M^2 \depthM{}, (M \depthM{})^{1.5} \})$, if the initial configuration MST has height \depth{} and maximum length $M$.
\end{lemma}

\textit{Sketch of Proof:}
After bounding the value of $R_t$, we upper bound the number of rounds for the part where $k_t \geq M$. 
The full proof is presented in \cite{techrep2018nf}.



We proceed with lower bounding the optimal offline strategy.
Consider $\mathcal{T}$, the MST on the initial configuration of the robots.
The longest edge length $M$ in $\mathcal{T}$ is a lower bound for the optimal solution, because no strategy with execution time less than $M$ can tag all the robots.
The following is the main result of our paper for arbitrary deployments.

\begin{theorem}
\label{thm:arb}
There exists an $O(\depthM{} \cdot \max\{M,\sqrt{M \depthM{}}\})$-competitive algorithm to the Online Network Formation Problem for arbitrary deployments in unbounded environments, where \depth{} is the height of the MST on the initial configuration and $M$ is the maximum edge length in the tree.
\end{theorem}

\begin{proof}
In Lemmas~\ref{lem:solSparse} and \ref{lem:solDense} we have shown that the time it takes to explore all the robots with \algo{} is $O(\max \{M^2 \depthM{}, (M \depthM{})^{1.5} \})$. Lower bounding the optimal solution by $M$ concludes the proof.
\end{proof}


\subsection{\algo{} in Random Deployments}
In this section, we present the performance of \algo{} for the case when the initial robot positions are chosen uniformly at random.
Suppose that the robots are distributed uniformly at random over an open circular area of radius $L$, which is unknown to the robots.

Because the environment radius is $L$, the furthest pairwise distance between the robots is no more than $2L$. We have a loose upper bound by noting that  all robots can be tagged by the first robot in time $4\pi L^2$.
Using \algo{}, however, we show that we can significantly improve the performance to $O(M^2)$.
We analyze the execution time of the algorithm in phases as in the case of arbitrary deployments.

\begin{proposition}
\label{prop:nRings}
Consider a partition of the environment surrounding the leader's initial position $x_1$ with $N$ circular rings of increasing radius $i M$, for $i = 1,\ldots, N$, until all robots are included.
Then, each ring has at least one robot in it and the number of rings is bounded as $N \leq 2L/M$.
\end{proposition}


In our analysis we use the expected number of robots in the rings. Let $S_i$ denote the number of robots in ring $i$, and $S_{1:i}$ be the number of robots up to and including the ring $i$.
To compute the expected values of $S_i$ and $S_{1:i}$, we first assume that the leader position is at the center of the area.

\begin{proposition}
\label{prop:nRobInRing}
Suppose the $i$-th innermost ring has radius $i M$ and there are $S_i$ robots inside. Then, the expected value of $S_i$ is lower bounded as $\mathbb{E}[S_i] \geq n (2i-1)M^2/4L^2$.
\end{proposition}


\begin{proposition}
\label{prop:nRobUpToRing}
The expected number of robots in up to and including the $i$-th innermost ring is lower bounded as $\mathbb{E}[S_{1:i}] \geq n (iM)^2/4L^2$.
\end{proposition}


The proofs of Propositions~\ref{prop:nRings}, \ref{prop:nRobInRing} and \ref{prop:nRobUpToRing} are given in \cite{techrep2018nf}.

\begin{lemma}
\label{lem:sol-unif}
Suppose $n$ robots are distributed uniformly at random over a circular area of radius $L$. The expected time to explore all robots is $O(L^2 \log n/n)$ using \algo{}.
\end{lemma}

\begin{proof}
From Proposition~\ref{prop:roundTime}, we know that the completion time of a round is no more than twice the size of a block in that round.
Since the $i$-th innermost ring has area $\pi M^2 (2i-1)$, we bound the expected execution time of \algo{} as follows.

\begin{equation}
\label{eq:solUnif}
\begin{split}
\mathbb{E}[SOL] & \leq \sum_{i=1}^{2L/M}{\frac{2\pi M^2 (2i-1)}{n \frac{i^2 M^2}{4L^2}}} \leq \sum_{i=1}^{2L/M}{16\pi \frac{L^2}{n i} } \\
	& \approx 16\pi L^2 \frac{\log (2L/M)}{n} \leq 16\pi L^2 \frac{\log n}{n},
\end{split}
\end{equation}

where the expectation is taken over all instances of the problem that are equally likely. 
The last inequality holds since $2L$ is the maximum possible pairwise distance and $M n \geq 2L$.
We can see that the exploration time is in $O(L^2 \log n/n)$.
Note that this bound does not change even when the leader is not at the center because the expected value of $S_{1:i}$ would be at least one fourth of the result in Proposition~\ref{prop:nRobUpToRing}.
\end{proof}


The longest edge length $M$ of the MST $\mathcal{T}$ determines an important property in initial configurations.
Suppose $G(n,r)$ denotes a Random Geometric Graph (RGG) of $n$ nodes each with communication radius $r$.
A two dimensional RGG is constructed by uniformly distributing $n$ points over a unit ball, and placing an edge between two nodes if they are within $r$ of each other.

An MST on a set of points is the spanning tree with the smallest maximum edge length.
Then, we can observe that $M$ is the critical distance to make $G(n,M)$ connected, because otherwise the MST would have a smaller longest edge.
The following result establishes a bound for the connectivity of an RGG.

\begin{lemma} (\cite{boyd2006randomized,gupta1999critical})
A two dimensional geometric random graph $G(n,r)$ is connected with probability at least $1 - 1/n^2$, if $r \geq \sqrt{2 \log n / n}$.
\label{lem:rgg}
\end{lemma}

Lemma~\ref{lem:rgg} is remarkable since it tells us that the critical distance to make an RGG connected is $M = O(L \sqrt{\log n/n})$, when the points are distributed over a circle of radius $L$.
%
%

\begin{theorem}
\label{thm:rnd}
There exists an $O(M)$-competitive algorithm to the \prob{} in unbounded environments for the case when robot initial positions are chosen uniformly at random.
\end{theorem}

\begin{proof}
In Lemma~\ref{lem:sol-unif}, we showed that the expected solution of \algo{} is $O(L^2 \log n/n)$. The proof follows by lower bounding the optimal solution as $M = O(L \sqrt{\log n/n})$.
\end{proof}

Theorem~\ref{thm:rnd} is our main result for the case of random deployments. 
In the next section, we present simulation results to validate the performance analysis.

\section{Simulations}
\label{sec:sim}
We conducted simulations in MATLAB to confirm our analytical analyses presented in the previous section.
The robots are deployed uniformly at random over a circular area of radius $L$ which is not accessible by the robots.
In the simulations we used varying number of robots from 20 to 1000.

\begin{figure}[ht!]
\centering
\includegraphics[width=0.5\columnwidth]{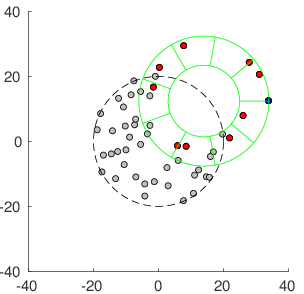}
\caption{An instance from the execution of \algo{} where 50 robots are deployed in an area of radius 20.}
\label{fig:simss}
\end{figure}

A snapshot from the execution of the algorithm is shown in Figure~\ref{fig:simss}. Here, the leader robot is shown in blue, the active robots are colored red, and the rest of the robots are gray.
The robots are distributed uniformly at random over a circular region of radius $L=20$ indicated by the dashed black circle.
The solid circles depict the blocks to be covered.


In Figure~\ref{fig:simres} we present the simulation results of our algorithm.
In this simulation, the robots are distributed uniformly at random over an area of radius $L=10$ and varying number of robots from $50$ up to $1000$ is used. For each number of robots we take the average of the algorithm's performance over 1000 trials.
We see that the $M^2 \depthM{}$ bound on arbitrary deployments fits nicely on random deployments as well. We note that the arbitrary deployments bound holds even without its coefficient derived in the analysis.
The expected bound for random deployments also fits the result especially after the number of robots is significantly large. For fewer number of robots the bound is loose, which suggests that there is a room for improvement in the analysis.

\begin{figure}[ht!]
\centering
\includegraphics[width=0.75\columnwidth]{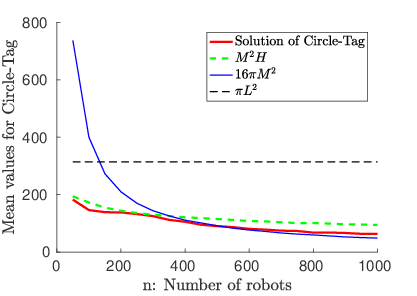}
\caption{Comparison between mean values of \algo{} performance and upper bounds for $L = 10$ and varying $n = 50$ to $1000$.}
\label{fig:simres}
\end{figure}

\section{Field Experiments}
\label{sec:exp}

In this section, we describe our multi-robot system and the field experiments where we demonstrated \algo{}.

\subsection{System Description}
The experiment requires a system of multiple robots, each able to navigate autonomously, communicate and make decisions on the fly.
Our system consists of three Uninhabited Aerial Vehicles (UAVs) and a ground station which is used to start and abort the mission.
In our experiments, we used the quadrotor DJI Matrice 100 equipped with a GPS and compass modules. As an on-board computation unit, each UAV uses the NVIDIA Jetson TX1 running Linux Ubuntu 16.04 operating system. The robots communication network is based on the XBee modules. Each robot and the base station has a XBee module so that the robot states can be monitored from the ground. This also allows the base station to interrupt the mission in case of emergency situations such as low battery voltage.


\subsection{State Machines for the Mission}
As part of the implementation of \algo{}, we designed finite state machines for the active and frozen robots.

Initially, the leader robot runs the state machine shown in Figure~\ref{fig:sm}(a) and all other robots run the one in \ref{fig:sm}(b).
In the beginning of the execution, all robots are in the state \texttt{IDLE} waiting for the ground station to start the mission.
The inactive robots remain in this state until they are found and became activated.
Once the leader hears from the station, it goes into the \texttt{TAKE-OFF} state. After the take-off, the leader is in the \texttt{GENERATE-WAYPOINTS} state where it computes a path generated by \algo{}.
While traveling on this path, the robot is in the \texttt{NAVIGATE} state. In this state, the active robot checks its vicinity using the XBee module and tags a frozen robot if it is nearby.
If there is a tagged robot by the end of a round, the newly activated robot starts executing the active finite state machine. The robot then takes-off and all active robots generate a new path.
This process is repeated until all robots are found and the round is finished.
After this point, the robots return to the home locations (a gathering point), and go into the \texttt{LAND} state to finalize the mission.


\begin{figure}
\centering     
\subfigure[Active strategy]{\includegraphics[width=0.5\columnwidth]{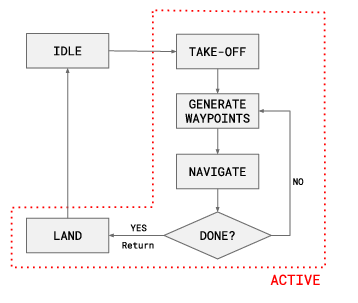}}
\subfigure[Frozen strategy]{\includegraphics[width=0.3\columnwidth]{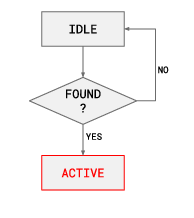}}
\caption{The finite state machines for the (a) active and (b) frozen robots}
\label{fig:sm}
\end{figure}

\subsection{Algorithm Demonstration}

We implemented the state machines in Figure~\ref{fig:sm} as Robot Operating System (ROS) based modules compatible with our hardware.
The experiments are conducted in Cedar Creek Ecosystem Science Reserve over a $100m \times 150m$ area.

Figure~\ref{fig:exp} shows the starting locations and trajectories of the robots used in the experiment.
Each solid line with a separate color indicates the trajectory of a robot, and the triangles are the starting positions of the robots. An instance of the experiment is illustrated in Figure~\ref{fig:uavFTP}.

Although our algorithm is designed for planar environments, we generalize it to the UAV setting as follows.
The highest UAV altitude is set to $h=40m$ and the robots increase the radius of their concentric-circles path by $2s = 60m$ at each circle. The communication range of the robots is $r=50m$. Then, any staying put robot is guaranteed to be explored by the search since if it has a distance $d$ to a hovering active robot, this distance will be at most $d \leq r = \sqrt{h^2 + s^2}$.

\begin{figure}[ht!]
\centering
\includegraphics[width=0.7\columnwidth]{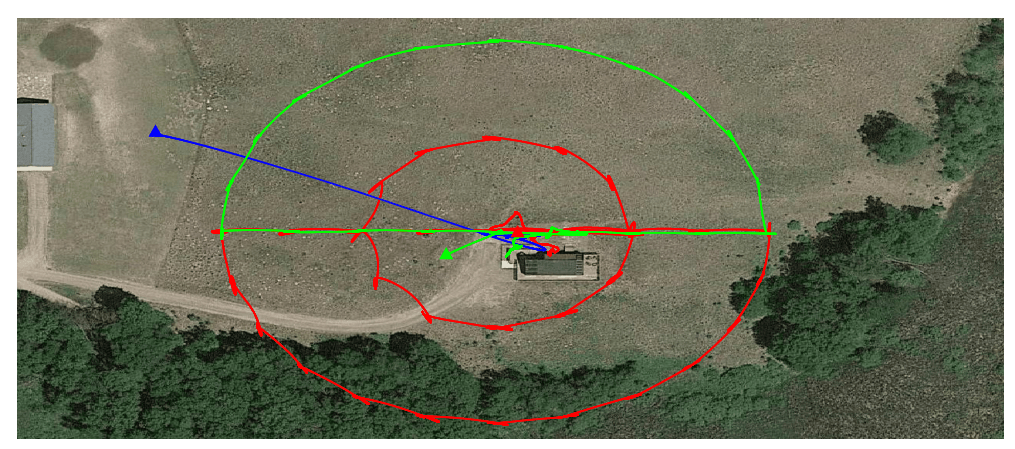}
\caption{Starting positions and trajectories of the UAVs during the mission performing \algo{}}
\label{fig:exp}
\end{figure}

\section{Conclusion}
\label{sec:conc}
In this paper, we studied the network formation problem for a multi-robot system operating in an unbounded area. Each robot has a limited sensing range, therefore only the robots in the vicinity are communicable.
For this problem, we introduced an algorithm whose competitive ratio is $O(\depthM{} \cdot \max\{M,\sqrt{M \depthM{}}\})$ for the case where the initial robot positions are chosen arbitrarily.
On the other hand, when the robots are distributed uniformly at random we improve the ratio to $O(M)$.
Future work includes removing the $O(\sqrt{M \depthM{}})$ term in the competitive ratio for the case of arbitrary deployments.
Another research direction is to consider the problem for a heterogeneous robots setting. In this case the task assignments of the robots can differ: for example, faster aerial robots can be used for exploring far away regions while the slower ground robots may be more efficient to search the nearby areas.

\bibliography{refs}
\bibliographystyle{plain}

\end{document}